\documentclass[journal]{IEEEtran}
\usepackage{epsfig,amsmath,amssymb,amsfonts,amstext,amsthm,mathrsfs}
\usepackage{epstopdf}
\usepackage{latexsym,algpseudocode,algorithm,algorithmicx,bm}
\usepackage{enumerate,cite}
\usepackage{graphicx,subfigure}
\usepackage{dsfont}
\usepackage{indentfirst,pdfsync,color}

\DeclareMathOperator*{\argmax}{arg\,max}
\DeclareMathOperator*{\argmin}{arg\,min}

\def\gT{{\mathcal{T}}}

\newtheorem{assum}{\textbf{Assumption}}
\newtheorem{definition}{\textbf{Definition}}

\newtheorem{lemma}{\textbf{Lemma}}
\newtheorem{theorem}{\textbf{Theorem}}
\newtheorem{proposition}{\textbf{Proposition}}
\newtheorem{remark}{\textbf{Remark}}

\newcommand{\nn}{\nonumber}

\DeclareMathAlphabet{\matheuf}{U}{euf}{m}{n}

\begin{document}
\title{Tightening Mutual Information Based Bounds on Generalization Error}
\author{Yuheng Bu, {\em  Member, IEEE},
Shaofeng Zou, {\em  Member, IEEE},
and Venugopal V. Veeravalli {\em Fellow, IEEE}\\
\thanks{This work was presented in part  at the IEEE International Symposium on Information Theory (ISIT), Paris, France, 2019  \cite{bu2019tightening}.}
\thanks{This work  was supported by Army Research Laboratory under Cooperative Agreement W911NF-17-2-0196, through the University of Illinois at Urbana-Champaign.}
\thanks{Y. Bu was  with the ECE Department and the Coordinated Science Laboratory, University of Illinois at Urbana-Champaign, Urbana, IL 61801 USA. He is now with the Institute for Data, Systems, and Society, Massachusetts Institute of Technology, Cambridge, MA 02139 USA (email: buyuheng@mit.edu).}
\thanks{S. Zou is with the Department of Electrical Engineering, University at Buffalo, The State University of New York, Buffalo, NY 14228 USA (email: szou3@buffalo.edu).}
\thanks{V. V. Veeravalli is with the ECE Department and the Coordinated Science Laboratory, University of Illinois at Urbana-Champaign, Urbana, IL 61801 USA (email: vvv@illinois.edu).}
}

\maketitle

\begin{abstract}
An information-theoretic upper bound on the generalization error of supervised learning algorithms is derived. The bound is constructed in terms of the mutual information between each individual training sample and the output of the learning algorithm. The bound is derived under more general conditions on the loss function than in existing studies; nevertheless, it provides a tighter characterization of the generalization error. Examples of learning algorithms are provided to demonstrate the the tightness of  the bound, and to show that it  has a broad range of applicability. Application to noisy and iterative algorithms, e.g., stochastic gradient Langevin dynamics (SGLD), is also studied, where the constructed bound provides a tighter characterization of the generalization error than existing results. Finally, it is demonstrated that, unlike existing bounds, which are difficult to compute and evaluate empirically, the proposed bound can be estimated easily in practice.
\end{abstract}


\section{Introduction}

Recent success of deep learning algorithms \cite{goodfellow2016deep} has dramatically boosted their applications in various engineering and science domains, e.g.,  computer vision \cite{krizhevsky2012imagenet}, natural language processing \cite{young2018recent}, autonomous driving \cite{huval2015empirical}, and health care \cite{miotto2018deep}.
A deep neural network trained using a sufficiently large amount of training data can achieve a small training error, while simultaneously performing well on unseen data, i.e., it generalizes well. However, we have yet to develop a satisfactory understanding of why deep learning algorithms generalize well. 

Classical statistical learning approaches for analyzing the generalization capability of supervised learning algorithms can be mainly categorized into two groups. The first set of methods are based on measures of the complexity of the output hypothesis space, e.g.,  VC dimension and Rademacher complexity \cite{boucheron2005theory,shalev2014understanding}. However, these complexity measures usually scale exponentially with the depth of  deep neural networks \cite{anthony2009neural}. Moreover, these approaches do not take into consideration the regularization implicitly imposed by the algorithms used to train the neural networks, e.g., stochastic gradient descent \cite{neyshabur2014search,zhang2016understanding}. Thus, the generalization error bounds based on these complexity measures tend to be loose and do not explain why deep neural networks generalize well in practice.  The second set of methods are based on exploiting properties of the learning algorithm, e.g., PAC-Bayesian bounds \cite{mcallester1999some}, uniform stability \cite{bousquet2002stability, elisseeff2005stability}, and compression bounds \cite{littlestone1986relating}. However, as discussed in \cite{zhang2016understanding,bartlett2017spectrally}, these approaches do not exploit the fact that  the generalization error  depends strongly on the underlying  true data-generating distribution. For example, if  the labels are irrelevant to the input features, then the generalization error will be large for a deep neural network, since training error is usually small due to the large capacity of the network, but test error will be large due to the fact that there is no relationship between the input features and the label \cite{zhang2016understanding}.
Recently, it was proposed in \cite{russo2016controlling} and further studied in \cite{xu2017information} and \cite{NIPS2018} that the metric of  mutual information can be used to develop upper bounds on the generalization error of learning algorithms. Such an information-theoretic framework can handle a broader range of problems, and it could also address the aforementioned challenges of implicit regularization and dependence on  data generating distribution.
More importantly, it offers an information-theoretic point of view on how to improve the generalization capability of a learning algorithm, and this new perspective provides us with a better understanding of the generalization behavior of deep neural networks.

In this paper, we follow the information-theoretic framework proposed in \cite{russo2016controlling,xu2017information,NIPS2018}. Our main contribution is a tighter upper bound on the generalization error using the mutual information between an individual training sample and the output hypothesis of the learning algorithm. We show that compared to existing studies, our bound has a broader range of applicability, and can be considerably tighter.




We consider an instance space $\mathcal{Z}$, a  hypothesis space $\mathcal{W}$, and a nonnegative loss function $\ell: \mathcal{W}\times \mathcal{Z} \to \mathbb{R}^+$. A training dataset $ S=\{Z_1,\cdots,Z_n\}$ that consists of $n$ i.i.d samples $Z_i \in \mathcal{Z}$ drawn from an unknown distribution $\mu$ is available. The goal of a supervised learning algorithm is to find an output hypothesis $w \in \mathcal{W}$ that minimizes the \emph{population risk}:
\begin{equation}
  L_\mu(w) \triangleq \mathbb{E}_{Z\sim \mu}[\ell(w,Z)]. 
\end{equation}
In practice, $\mu$ is unknown, and thus $L_\mu(w)$ cannot be computed directly. Instead,  the \emph{empirical risk} of $w$ on a training dataset $S$ is studied, which is defined as
\begin{equation}
  L_S(w)\triangleq \frac{1}{n} \sum_{i=1}^n \ell(w,Z_i).
\end{equation}
A learning algorithm can be characterized by a randomized mapping from the training dataset $S$ to a hypothesis $W$ according to a conditional distribution $P_{W|S}$.

In statistical learning theory, the \emph{(mean) generalization error}\footnote{ The term ``generalization error" of a learning algorithm is usually defined as the difference between the population risk and the training error \emph{without} taking the expectation with respect to the randomness of the data and the learning algorithm. Here, we consider the expectation of the generalization error over the randomness of both the data and the learning algorithm. We will use the term ``generalization error" throughout the paper, with the understanding that it is  the ``mean generalization error".} of a supervised learning algorithm is the expected difference between the population risk of the output hypothesis and its empirical risk on the training dataset:
\begin{equation}
  \mathrm{gen}(\mu,P_{W|S}) \triangleq \mathbb{E}_{W, S}[L_\mu(W)-L_S(W)],
\end{equation}
where the expectation is taken over the joint distribution $P_{S,W} = P_S\otimes P_{W|S}$. Note that $P_{W|S}$ will become degenerate if $W$ is a deterministic function of $S$. 
The generalization error is used to measure the extent to which the learning algorithm overfits the training data.

\subsection*{Main Contributions and Related Works }
We first review the following lemma from \cite{xu2017information}, which provides an upper bound on the generalization error using the mutual information $I(S;W)$ between the training dataset $S$ and the output hypothesis $W$.
\begin{lemma}[{\cite[Theorem 1]{xu2017information}}]\label{lemma:original}
Suppose $\ell(w,Z)$ is $R$-sub-Gaussian\footnote{A random variable $X$ is $R$-sub-Gaussian if $\log \mathbb{E}[e^{\lambda(X-\mathbb{E}X)}] \le \frac{R^2\lambda^2}{2}$, $\forall \lambda \in \mathbb{R}$.} under $Z\sim \mu$ for all $w \in \mathcal{W}$, then
\begin{equation}\label{eq:4}
  |\mathrm{gen}(\mu,P_{W|S}) | \le \sqrt{\frac{2R^2}{n} I(S;W)}.
\end{equation}
\end{lemma}

This mutual information based bound in \eqref{eq:4} is related to the ``on-average'' stability (see, e.g., \cite{shalev2010learnability}), since it quantifies the overall dependence between the output of the learning algorithm and all the input training samples  via $I(S;W)$. Note that $I(S;W)$ depends on the main components of a supervised learning problem, i.e., the hypothesis space $\mathcal{W}$, the learning algorithm $P_{W|S}$, and the data generating distribution $\mu$, in contrast to the traditional bounds based on VC dimension or the uniform stability, which only depend on one aspect of the learning problem. We also note that there is a  connection between the mutual information based generalization bound and the PAC-Bayesian bound  in \cite{bassily2017learners}, since both methods adopt the variational representation of relative entropy to establish the decoupling lemma. By further exploiting the structure of the hypothesis space and the dependency between the algorithm input and output, the authors of \cite{NIPS2018,asadi2019chaining} combined the chaining and mutual information methods, and obtained a tighter bound on the generalization error.

However, the bound in Lemma \ref{lemma:original} and the chaining mutual information (CMI) bound in \cite{NIPS2018} both suffer from the following two shortcomings.
First, for empirical risk minimization (ERM), if $W$ is the unique minimizer of $L_S(w)$ in $\mathcal{W}$, then $W$ is a deterministic function of $S$ and the mutual information $I(S;W)=\infty$. It can be shown that both bounds are not tight in this case. Second, both bounds assume that $\ell(w,Z)$ has a bounded cumulant generating function (CGF) under $Z\sim \mu$ \emph{for all} $w \in \mathcal{W}$, which may not hold in many cases (see Section \ref{sec:deterministic}).

There has been some recent work on addressing these shortcomings of  mutual information based bounds on generalization error by using other information-theoretical measures, e.g., Wasserstein distance \cite{raginsky2016information,lopez2018generalization, ISIT2019wass}, maximal leakage \cite{issa2018computable,issa2019strengthened} and total variation \cite{alabdulmohsin2015algorithmic} to bound the generalization error. But the measures proposed in these papers are difficult to evaluate both analytically and empirically as we discuss in Section \ref{sec:eva}, which significantly undermines the usefulness of these results in practice.





In this paper, we get around the aforementioned shortcomings by combining the idea of point-wise stability \cite{elisseeff2005stability,raginsky2016information} with the information-theoretic framework introduced in \cite{xu2017information}. Specifically, an algorithm is said to be point-wise stable if the expectation of the loss function $\ell(W,Z_i)$ does not change too much with the replacement of any \textit{individual} training sample $Z_i$, and if an algorithm is point-wise stable, then it generalizes well \cite{elisseeff2005stability,raginsky2016information}. Motivated by these facts, we tighten the mutual information based generalization error bound through a bound based on  the individual sample mutual information (ISMI) $I(W;Z_i)$. Compared with the bound in Lemma \ref{lemma:original}, and the CMI bound in \cite{NIPS2018}, the ISMI bound is derived under a more general condition on the CGF of the loss function, is applicable to a broader range of problems, and can provide a tighter characterization of the generalization error.

The rest of the paper is organized as follows. In Section \ref{sec:pre}, we provide some preliminary definitions and results for our analysis. In Section \ref{sec:main}, we  introduce the individual sample mutual information generalization bound.  In Section \ref{sec:deterministic}, we apply our method to bound the generalization errors of two learning problems with infinite $I(S;W)$. We show in the second example that our ISMI bound can be tighter than the CMI bound in \cite{NIPS2018}, while the bound in Lemma \ref{lemma:original} is infinity. In Section \ref{sec:noisy}, we improve the generalization error bound in \cite{Pensia2018ISIT} for SGLD algorithm using our method, which demonstrates that the ISMI bound is applicable to the noisy, iterative algorithms discussed in \cite{Pensia2018ISIT}. In Section \ref{sec:eva}, we provide an example where the ISMI bound can be evaluated empirically from the samples, while other existing bounds are difficult to estimate due to prohibitive computational complexity.

\section{Preliminaries}\label{sec:pre}
We use upper letters to denote random variables, and calligraphic upper letters to denote sets. For a random variable $X$ generated from a distribution $\mu$, we use $\mathbb{E}_{X\sim\mu}$ to denote the expectation taken over $X$ with distribution $\mu$.  We write $I_d$ to denote the $d$-dimensional identity matrix. All the logarithms are the natural ones, and all the information measure units are nats. We use $\mu^{\otimes n}$ to denote the product distribution of $n$  copies of $\mu$.

\begin{definition}
The cumulant generating function (CGF) of a random variable $X$ is defined as
\begin{equation}
\Lambda_X(\lambda) \triangleq \log \mathbb{E}[e^{\lambda(X-\mathbb{E}X)}].
\end{equation}
\end{definition}
Assuming $\Lambda_X(\lambda)$ exists, it can be verified that $\Lambda_X(0)=\Lambda_X'(0)=0$, and that it is convex.

\begin{definition}
For a convex function $\psi$ defined on the interval $[0,b)$, where $0<b\le \infty$, its Legendre dual $\psi^*$ is defined as
\begin{equation}
  \psi^*(x) \triangleq \sup_{\lambda \in [0,b)} \big(\lambda x-\psi(\lambda)\big).
\end{equation}
\end{definition}

The following lemma characterizes a useful property of the Legendre dual and its inverse function.
\begin{lemma}[{\cite[Lemma 2.4]{boucheron2013concentration}}] \label{lemma:psi_star}
Assume that $\psi(0)= \psi'(0) = 0$. Then $\psi^*(x)$ defined above is a non-negative convex and non-decreasing function on $[0,\infty)$ with $\psi^*(0) = 0$. Moreover, its inverse function $\psi^{*-1}(y) = \inf\{x\ge 0: \psi^*(x)\ge y\}$ is concave, and can be written as
\begin{equation}
  \psi^{*-1}(y) = \inf_{\lambda \in (0,b)} \Big( \frac{y+\psi(\lambda)}{\lambda} \Big).
\end{equation}
\end{lemma}
For an $R$-sub-Gaussian random variable $X$, $\psi(\lambda)= \frac{R^2\lambda^2}{2}$ is an upper bound on $\Lambda_X(\lambda)$. Then by Lemma \ref{lemma:psi_star}, $\psi^{*-1}(y)=\sqrt{2R^2y}$.

\section{Bounding Generalization Error via $I(W;Z_i)$}\label{sec:main}
In this section, we first generalize the decoupling lemma in \cite[Lemma 1]{xu2017information} to a different setting, and then tighten the bound on generalization error via the individual sample mutual information $I(W;Z_i)$.



\subsection{General Decoupling Estimate}
Consider a pair of random variables $W$ and $Z$ with joint distribution $P_{W,Z}$. Let $\widetilde{W}$ be an independent copy of $W$, and $\widetilde{Z}$ be an independent copy of $Z$, such that $P_{\widetilde{W}\widetilde{Z}} = P_{{W}} \otimes P_{{Z}}$. Suppose $f: \mathcal{W}\times  \mathcal{Z} \to \mathbb{R}$ is a real-valued function. If the CGF $\Lambda_{f(\widetilde{W},\widetilde{Z})}(\lambda)$ of $f(\widetilde{W},\widetilde{Z})$ can be upper bounded by some function $\psi$ for $\lambda\in(b_-,b_+)$, we have the following theorem.
\begin{theorem}\label{thm:decouple}
Assume that $\Lambda_{f(\widetilde{W},\widetilde{Z})}(\lambda) \le \psi_{+}(\lambda)$  for $\lambda \in [0,b_+)$, and $\Lambda_{f(\widetilde{W},\widetilde{Z})}(\lambda) \le \psi_{-}(-\lambda)$ for $\lambda\in (b_-,0]$ under distribution $P_{\widetilde{W}\widetilde{Z}} = P_{{W}} \otimes P_{{Z}}$, where $0 <b_+\le \infty$ and $-\infty \le b_- <0 $. Suppose that $\psi_{+}(\lambda)$ and $\psi_{-}(\lambda)$ are convex, and  $\psi_{+}(0) =  \psi_{+}'(0) = \psi_{-}(0) =  \psi_{-}'(0)=0$. Then,
\begin{align}
  \mathbb{E}[f(W,Z)] - \mathbb{E}[f(\widetilde{W},\widetilde{Z})]  &\le \psi^{*-1}_{+} \big(I(W;Z)\big),\\
    \mathbb{E}[f(\widetilde{W},\widetilde{Z})] - \mathbb{E}[f(W,Z)] &\le \psi^{*-1}_{-} \big(I(W;Z)\big).
\end{align}
\end{theorem}

\begin{proof}
Consider the  variational representation of the relative entropy between two probability measures $P$ and $Q$ defined on $\mathcal{X}$:
\begin{equation}
  D(P\|Q) = \sup_{g \in \mathcal G} \Big\{ \mathbb{E}_P[g(X)]  -\log \mathbb{E}_Q[e^{g(X)}] \Big\},
\end{equation}
where the supremum is over all measurable functions $\mathcal G =\{g : \mathcal{X} \to \mathbb{R},\ \mathrm{s.t.}\ \mathbb{E}_Q[e^{g(X)}]< \infty\}$, and equality is achieved when $g=\log \frac{dP}{dQ}$, where $\frac{dP}{dQ}$ is the Radon--Nikodym derivative. It then follows that   $\forall\lambda \in [0,b_+)$,
\begin{align}\label{equ:positive}
 I(W;Z) &= D(P_{W,Z}\| P_W \otimes P_Z ) \nn \\
 &\ge \mathbb{E}[\lambda f(W,Z)] - \log \mathbb{E}[e^{\lambda f(\widetilde{W},\widetilde{Z})}]\nn\\
  &\ge \lambda(\mathbb{E}[ f(W,Z)] - \mathbb{E}[ f(\widetilde{W},\widetilde{Z})])-\psi_{+}(\lambda),
\end{align}
where the last inequality follows from the assumption that
\begin{equation}
\Lambda_f (\widetilde{W}; \widetilde{Z}) = \log\mathbb{E}[e^{\lambda(f(\widetilde{W},\widetilde{Z})-\mathbb{E}f(\widetilde{W},\widetilde{Z}))}] \le \psi_{+}(\lambda),
\end{equation}
for $\lambda \in [0,b_+)$. Similarly, for $\lambda\in (b_-,0]$, it follows that
\begin{align}\label{equ:negative}
  D(P_{W,Z}\|&P_W \otimes P_Z )  \nn \\
  &\ge \lambda(\mathbb{E}[ f(W,Z)] - \mathbb{E}[ f(\widetilde{W},\widetilde{Z})])-\psi_{-}(-\lambda).
\end{align}

From \eqref{equ:positive} it follows that
\begin{align}\label{eq:thm1_proof1}
\mathbb{E}[ f(W,Z)] - \mathbb{E}[ f(\widetilde{W},\widetilde{Z})] & \le \inf_{\lambda \in [0,b_+)} \frac{I(W;Z)+\psi_{+}(\lambda)}{\lambda}\nn \\
&= \psi^{*-1}_+\big(I(W;Z)\big),
\end{align}
and from \eqref{equ:negative} it follows that
\begin{align}\label{eq:thm1_proof2}
\mathbb{E}[ f(\widetilde{W},\widetilde{Z})] - \mathbb{E}[ f(W,Z)]
& \le \inf_{\lambda \in [0,-b_-)} \frac{I(W;Z)+\psi_{-}(\lambda)}{\lambda}\nn \\
&= \psi^{*-1}_{-}\big(I(W;Z)\big),
\end{align}
where the equalities in \eqref{eq:thm1_proof1} and \eqref{eq:thm1_proof2} follow from Lemma \ref{lemma:psi_star}.
\end{proof}

Theorem \ref{thm:decouple} provides a different characterization of the decoupling estimate than existing results. Specifically, it is assumed that the CGF of $f(w,Z)$ is bounded for all $w\in \mathcal{W}$ and $Z\sim \mu$ in \cite[Lemma 1]{xu2017information} and \cite[Theorem 2]{jiao2017dependence}, whereas in Theorem \ref{thm:decouple}, it is assumed that the CGF of $f(\widetilde W,\widetilde Z)$ is bounded in expectation under $P_{\widetilde{W}\widetilde{Z}} = P_{{W}} \otimes P_{{Z}}$.




\subsection{Individual Sample Mutual Information Bound}
Motivated by the idea of algorithmic stability, which measures how much an output hypothesis changes with the replacement of an \textit{individual} training sample, we construct the following upper bound on the generalization error via $I(W;Z_i)$.


\begin{theorem}\label{thm:main}
Suppose $\ell(\widetilde W,\widetilde Z)$ satisfies $\Lambda_{\ell(\widetilde W,\widetilde Z)}(\lambda) \le \psi_{+}(\lambda)$ for $\lambda \in [0,b_+)$, and $\Lambda_{\ell(\widetilde W,\widetilde Z)}(\lambda) \le \psi_{-}(-\lambda)$ for $\lambda\in (b_-,0]$ under $P_{\widetilde Z,\widetilde W} = \mu\otimes P_{W}$, where $0 <b_+\le \infty$ and $-\infty \leq b_- <0 $. Then,
\begin{align}
    \mathrm{gen}(\mu,P_{W|S}) \le \frac{1}{n} \sum_{i=1}^n  \psi^{*-1}_{-}\big(I(W;Z_i)\big),\\
  -\mathrm{gen}(\mu,P_{W|S}) \le \frac{1}{n} \sum_{i=1}^n  \psi^{*-1}_{+}\big(I(W;Z_i)\big).
\end{align}
\end{theorem}

\begin{proof}

The generalization error can be written as follows:
\begin{align}\label{eq:gen_dec}
  \mathrm{gen}(\mu,P_{W|S})
   & = \frac{1}{n} \sum_{i=1}^n \Big(\mathbb{E}_{W,\widetilde{Z}}[\ell(W,\widetilde{Z})] - \mathbb{E}_{W,Z_i} [\ell(W,Z_i)] \Big),
\end{align}
where $W$ and $Z_i$ in the second term are dependent with $P_{W,Z_i} = \mu\otimes P_{W|Z_i}$, and $W$ and $\widetilde{Z}$ in the first term are independent with the same marginal distributions. Applying Theorem \ref{thm:decouple}  completes the proof.
%
\end{proof}

In the following proposition, we derive the ISMI bounds under two different sub-Gaussian assumptions.
\begin{proposition}\label{prop:ISMI}
\begin{enumerate}
\item Suppose that $\ell(w,Z)$ is $R$-sub-Gaussian under $Z\sim \mu$ for all $w\in \mathcal W$, then
\begin{equation}
  \big|\mathrm{gen}(\mu,P_{W|S})\big| \le \frac{1}{n} \sum_{i=1}^n  \sqrt{2R^2I(W;Z_i)}.
\end{equation}
\item Suppose that $\ell(\widetilde{W},\widetilde{Z})$ is $R$-sub-Gaussian under distribution $P_{\widetilde{W}\widetilde{Z}} = P_{{W}} \otimes P_{{Z}}$, then
\begin{equation}
  \big|\mathrm{gen}(\mu,P_{W|S})\big| \le \frac{1}{n} \sum_{i=1}^n  \sqrt{2R^2I(W;Z_i)}.
\end{equation}
\end{enumerate}
\end{proposition}

\begin{proof}
1) The generalization error can be written as in \eqref{eq:gen_dec}, where $W$ and $Z_i$ in the second term are dependent with $P_{W,Z_i} = \mu\otimes P_{W|Z_i}$, and $W$ and $\widetilde{Z}$ in the first term are independent whose marginal distributions are  the same as those of $W$ and $Z_i$. The first inequality then follows from Lemma \ref{thm:decouple} by letting $S=Z_i$ and $n=1$, for each $i=1,\cdots,n$.

2) 
For an $R$-sub-Gaussian random variable, $\psi^{-1}_+(y) = \psi^{-1}_-(y) =\sqrt{2R^2 y}$ is an upper bound on its CGF. The second inequality then follows from Theorem \ref{thm:main}.
\end{proof}

\begin{remark}
The condition that $\ell(w,Z)$ is $R$-sub-Gaussian under $Z\sim \mu$ for all $w \in \mathcal{W}$ in the first part of Proposition \ref{prop:ISMI} is the same as the one in Lemma \ref{lemma:original}, which is not stronger than the condition in the second part of Proposition \ref{prop:ISMI}. An example was given in \cite{negrea2019information} for this argument. Specifically, consider  $\mathcal{W} = \mathcal{Z} = \mathbb{R}$, with $\ell(w,z) = w+z$, and $(\widetilde{W},\widetilde{Z}) \sim \text{Cauchy} \otimes \mathcal{N}(0,1)$. Then, $\ell(w,Z)$ is 1-sub-Gaussian for any $w \in \mathcal{W}$, whereas $\ell(\widetilde{W},\widetilde{Z})$ does not even have bounded absolute first moment.
\end{remark}

The following proposition shows that the proposed ISMI bound is always no worse than the bound using $I(S;W)$ in Lemma \ref{lemma:original} and \cite[Theorem 2]{jiao2017dependence}. 
\begin{proposition}\label{prop:tight}
Suppose that $S=\{Z_1,\cdots,Z_n\}$ consists of $n$ independent samples, and $\psi^{*-1}$ is a concave function, then
\begin{equation}
   \frac{1}{n} \sum_{i=1}^n  \psi^{*-1}\bigg(I(W;Z_i)\bigg) \le \psi^{*-1}\bigg(\frac{I(S;W)}{n}\bigg).
\end{equation}
\end{proposition}

\begin{proof}
By the chain rule of mutual information, 
\begin{align}
  I(W;S) = \sum_{i=1}^n I(W;Z_i|Z^{i-1})
\end{align}
where $Z^{j}=\{Z_1,\cdots,Z_{j}\}$.
Note that $Z_i$ and $Z^{i-1}$ are independent, i.e., $I(Z_i;Z^{i-1})=0$ , it then follows that
\begin{align}
  I(W;Z_i|Z^{i-1})  &= I(W;Z_i|Z^{i-1}) + I(Z_i;Z^{i-1}) \nn \\
  &= I(W,Z^{i-1};Z_i) \nn \\
  &= I(W;Z_i)+I(Z^{i-1};Z_i|W)  \nn \\
  &\ge I(W;Z_i).
\end{align}
Thus,
\begin{align}
  I(W;S) &= \sum_{i=1}^n I(W;Z_i|Z^{i-1})
  \ge \sum_{i=1}^n I(W;Z_i),
\end{align}
and applying Jensen's inequality completes the proof.
\end{proof}

\begin{remark} Under the sub-Gaussian condition (see sentence following Lemma \ref{lemma:original}), we can let $ \psi^{*-1}\big(y\big)=\sqrt{2R^2y}$. Then by Proposition \ref{prop:tight}, the ISMI bound in Proposition \ref{prop:ISMI} is always no worse than the bound based on $I(S;W)$ in Lemma \ref{lemma:original}.
\end{remark}
\begin{remark}
Following  arguments similar to those  used in the proof of $I(W;Z_i) \le I(W;Z_i|Z^{i-1})$, we can also show that $I(W;Z_i) \le I(W;Z_i|S^{-i})$, where $S^{-i}$ denotes the set obtained by deleting $Z_i$ from $S$. Therefore, the ISMI bound is always no worse than the bound based on  $I(W;Z_i|S^{-i})$ in \cite[Theorem 2]{raginsky2016information}. 
\end{remark}

In the next section, we will also show via several examples that the ISMI bound provides a more accurate characterization of the generalization error than the bound in Lemma \ref{lemma:original}  and the chaining bound in \cite{NIPS2018}.

\section{Examples with Infinite $I(W;S)$}\label{sec:deterministic}
In this section, we consider two examples of learning algorithms with infinite $I(W;S)$. We show that for these  examples, the upper bound on generalization error in Lemma \ref{lemma:original} blows up, whereas the ISMI bound in Theorem \ref{thm:main} still provides an accurate approximation. The details of the derivations of the bounds can be found in the Appendices.


\subsection{Estimating the Mean}\label{sec:mean}
We first consider the problem of learning the mean of a Gaussian random vector $Z \sim \mathcal{N}(\mu, \sigma^2 I_d)$, which minimizes the  square error $\ell(w,Z)\triangleq \|w-Z\|_2^2$.
The empirical risk with $n$ i.i.d. samples is
\begin{equation}
L_S(w)\triangleq\frac{1}{n} \sum_{i=1}^n \|w-Z_i\|_2^2, \quad w \in \mathbb{R}^d.
\end{equation}
The empirical risk minimization (ERM) solution is  the sample mean $W = \frac{1}{n} \sum_{i=1}^n Z_i$, which is deterministic given $S$. Its generalization error can be computed exactly as (see Appendix \ref{app:mean}):
\begin{align}\label{eq:mean_truth}
  \mathrm{gen}(\mu,P_{W|S}) 
   & = \frac{2\sigma^2 d}{n}.
\end{align}

The bound in Lemma \ref{lemma:original} is not applicable here due to the following two reasons:  (1)  $W$ is a deterministic function of $S$, and hence $I(S;W)=\infty$; and (2) since  $Z$ is a Gaussian random vector, the loss function $\ell(w,Z)=\|w-Z\|_2^2$ is not sub-Gaussian for all $w \in \mathbb{R}^d$. Specifically, the variance of the loss function $\ell(w,Z)$ diverges as $\|w\|_2 \to \infty$, which implies that a uniform upper bound on $\Lambda_{\ell(w,Z)}(\lambda)$, $\forall w\in \mathbb{R}^d$  does not exist.

We can get around both of these issues by applying the ISMI bound in Theorem \ref{thm:main}. Since $W\sim \mathcal{N}(\mu, \frac{\sigma^2I_d}{n})$, the mutual information between each individual sample and the output hypothesis $I(W;Z_i)$ can be computed exactly as (see Appendix \ref{app:mean}):
\begin{align}\label{eq:mean_MI}
  I(W;Z_i) 
  & = \frac{d}{2} \log \frac{n}{n-1}, \qquad i=1,\cdots,n, \quad n\ge 2.
\end{align}
In addition, since $W\sim \mathcal{N}(\mu, \frac{\sigma^2I_d}{n})$, it can be shown that $\ell(W,\widetilde{Z}) \sim \sigma_\ell^2 \chi^2_d$, where $\sigma_\ell^2 \triangleq \frac{(n+1)\sigma^2}{n}$, and $\chi^2_d$ denotes the chi-squared distribution with $d$ degrees of freedom. Note that the expectation of $\chi^2_d$ distribution is $d$ and its moment generating function is $(1-2\lambda)^{d/2}$. Therefore, the CGF of $\ell(\widetilde W,\widetilde Z)$ is given by
\begin{equation}
  \Lambda_{\ell(\widetilde W,\widetilde Z)}(\lambda) = - d \sigma_\ell^2 \lambda - \frac{d}{2} \log(1-2\sigma_\ell^2\lambda),
\end{equation}
for $\lambda \in (-\infty, \frac{1}{2\sigma_\ell^2})$.
Since $W$ is the ERM solution, it follows that $\mathrm{gen}(\mu,P_{W|S})\ge 0$, and we only need to consider the case $\lambda <0$. It can be shown that (see Appendix \ref{app:mean}):
\begin{equation}
  \Lambda_{\ell(\widetilde W,\widetilde Z)}(\lambda) \le d\sigma_\ell^4\lambda^2 \triangleq \psi_-(-\lambda),\quad \lambda<0.
\end{equation}
Then, $\psi_-^{*-1}(y) = 2 \sqrt{d\sigma_\ell^4 y}.$ 
Combining the results in \eqref{eq:mean_MI}, we have
\begin{align}
  \mathrm{gen}(\mu,P_{W|S})  \le  \sigma^2 d \sqrt{\frac{2(n+1)^2}{n^2} \log \frac{n}{n-1} }.
\end{align}
As $n \to \infty$, the above bound is $\mathcal{O}\left(\frac{1}{\sqrt{n}}\right)$, which is sub-optimal compared to the true generalization error computed in \eqref{eq:mean_truth}. We should note that techniques based on VC dimension \cite{boucheron2005theory} and algorithmic stability \cite{bousquet2002stability} also yield bounds of $\mathcal{O}\left(\frac{1}{\sqrt{n}}\right)$.

\subsection{Gaussian Process}\label{sec:GP}
In this subsection, we revisit the Gaussian process example studied in \cite{NIPS2018}. Let $\mathcal{W} = \{w \in \mathbb{R}^2:\|w\|_2=1 \}$, and $Z \sim \mathcal{N}(0, I_2)$ be a standard normal random vector in $\mathbb{R}^2$. The loss function is defined to be the following Gaussian process indexed by $w$:
\begin{equation}
  \ell(w,Z) \triangleq -\langle w,  Z\rangle,\quad  \forall w\in \mathcal{W}.
\end{equation}
Note that the loss function\footnote{The loss function can be negative here. We ignore the non-negativity assumption of the loss function; this does not affect our analysis.} $\ell(w,Z)$  is sub-Gaussian with parameter $R=1$ for all $w\in \mathcal{W}$. In addition, the output hypothesis $w\in \mathcal{W}$  can also be represented equivalently  using the phase of $w$. In other words, we can let $\phi$ be the unique number in $[0, 2\pi)$ such that $w = (\sin \phi, \cos \phi)$.
For this problem, the empirical risk of a hypothesis $w\in \mathcal{W}$ is given by
\begin{equation}
L_S(w) = - \frac{1}{n}\sum_{i=1}^n \langle w,Z_i\rangle.
\end{equation}
We consider two learning algorithms which are the same as the ones in \cite{NIPS2018}. The first is the ERM algorithm:
\begin{equation}
  W = \argmin_{\phi\in [0,2\pi)} L_S(w) = \argmax_{\phi\in [0,2\pi)} \big\langle w,  \frac{1}{n}\sum_{i=1}^n Z_i \big\rangle.
\end{equation}
The second is the ERM algorithm with additive noise:
\begin{equation}
  W' = \Big(\argmax_{\phi\in [0,2\pi)} \big\langle w,  \frac{1}{n}\sum_{i=1}^n Z_i\big\rangle \Big)\oplus \xi \  (\bmod\ 2\pi),
\end{equation}
where the noise $\xi$ is independent of $S$, and has an atom with probability mass $\epsilon$ at 0, and  probability $1- \epsilon$ uniformly distributed on $(-\pi,\pi)$. Due to the symmetry of the problem, $W$ and $W'$ are uniformly distributed over $[0, 2\pi)$.

The generalization error of the ERM algorithm $W$ can be computed exactly as (see Appendix \ref{app:GP}):
\begin{align}
  \mathrm{gen}(\mu,P_{W|S}) 
  =\sqrt{\frac{\pi}{2n}}.
\end{align} 
For the second algorithm $W'$, since the noise $\xi$ is independent from $S$, it follows that
\begin{equation}
  \mathrm{gen}(\mu,P_{W'|S})
  =\epsilon\sqrt{\frac{\pi}{2n}}.
\end{equation}

The  bound via $I(W;S)$ in Lemma \ref{lemma:original} is not applicable, since $W$ is deterministic given $S$ and $I(W;S)=\infty$.
Moreover, for the second algorithm $W'$,
\begin{align}
  I(W';S) &= h(W')- h(W'|S) \nn \\
  &=\log 2\pi - h(\xi) =\infty,
\end{align}
since $\xi$ has a singular component at 0, and $h(\xi) = -\infty$.



Applying the ISMI bound in Theorem \ref{thm:main} to the ERM algorithm $W$, we have that
\begin{align}\label{eq:GP_newbound}
  I(W;Z_i) &= h(W)- h(W|Z_i) \nn \\
  & = \log 2\pi - h(W|Z_i)\nn \\
  &=\log 2\pi - \mathbb{E}_{Z_i}[h(W|Z_i=z_i)],
\end{align}
which we need to compute the conditional distribution $P_{W|Z_i=z_i}$.
Note that given $Z_i=z_i$, the ERM solution
\begin{equation}
W = \argmax_{\phi\in [0,2\pi)} \langle w,  \frac{z_i}{n}+\frac{1}{n}\sum_{j\ne i} Z_i\rangle
\end{equation}
depends on the other samples $Z_j$, $j\ne i$.  Moreover, it can be shown that   $P_{W|Z_i=z_i}$ is equivalent to the phase distribution of a Gaussian random vector $\mathcal{N}(\frac{z_i}{n}, \frac{n-1}{n^2}I_2)$ in polar coordinates.

Due to symmetry, we can always rotate the polar coordinates, such that $z_i=(r,0)$, where $r \in \mathbb{R}^+$ is the $\ell_2$ norm of $z_i$. Then, $P_{W|Z_i=z_i}$ is a function of $r$, and can be equivalently characterized and computed by the distribution $f\big(\phi\big|\|Z_i\|=r\big)$ provided in Appendix \ref{app:GP}.
Since the norm of $Z_i$ has a Rayleigh distribution with unit variance, it then follows that
\begin{align}
  I(W;Z_i) 
  =\log 2\pi - \mathbb{E}_{\|Z_i\|}\Big[h\big(f(\phi\big|\|Z_i\|=r)\big)\Big].
\end{align}

Applying Theorem \ref{thm:main}, we obtain
\begin{equation}
  |\mathrm{gen}(\mu,P_{W|S}) |\le \frac{1}{n}\sum_{i=1}^n\sqrt{2 I(W;Z_i) }=\sqrt{2 I(W;Z_i) }.
\end{equation}
Similarly, we can compute the ISMI bound for $W'$.


\begin{figure}
	\centering
	\includegraphics[width=0.9\columnwidth]{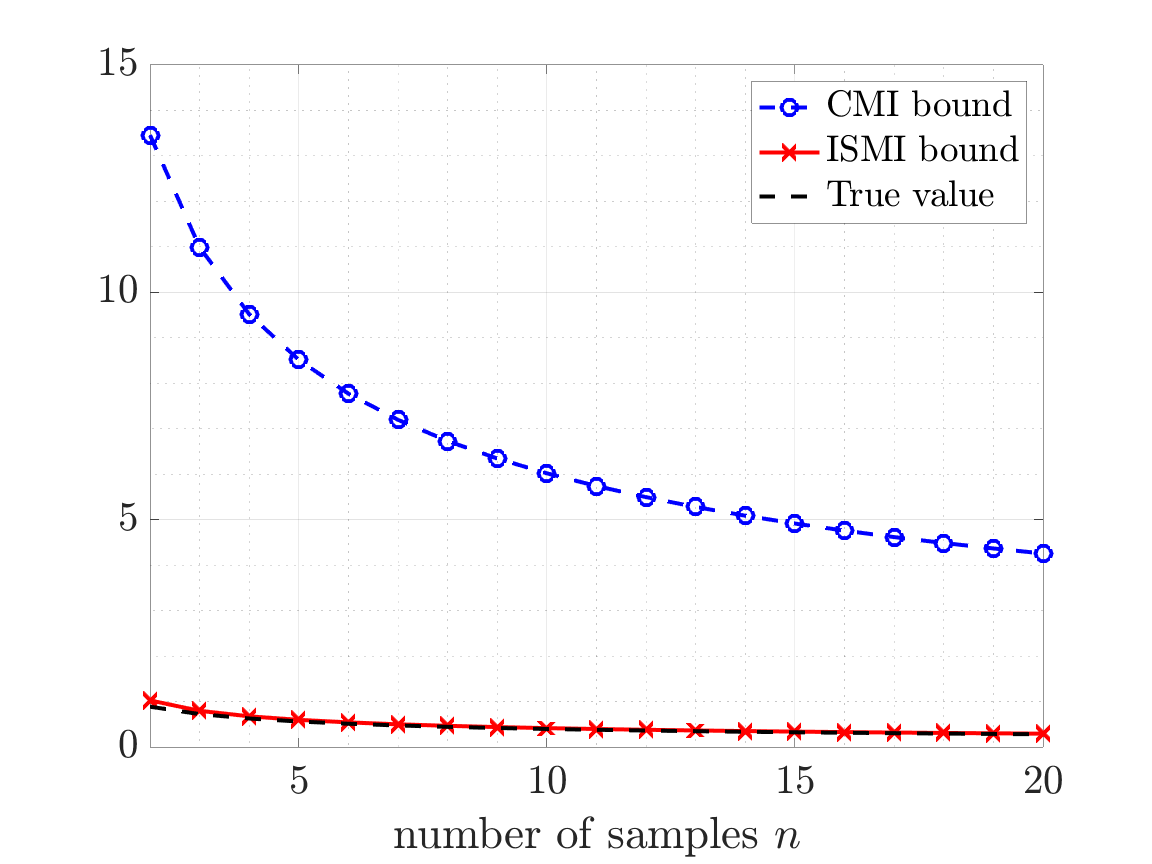}\\
	\caption{Comparison of generalization bounds for the ERM algorithm. }\label{fig:ERM}
\end{figure}
\begin{figure}
	\centering
	\includegraphics[width=0.9\columnwidth]{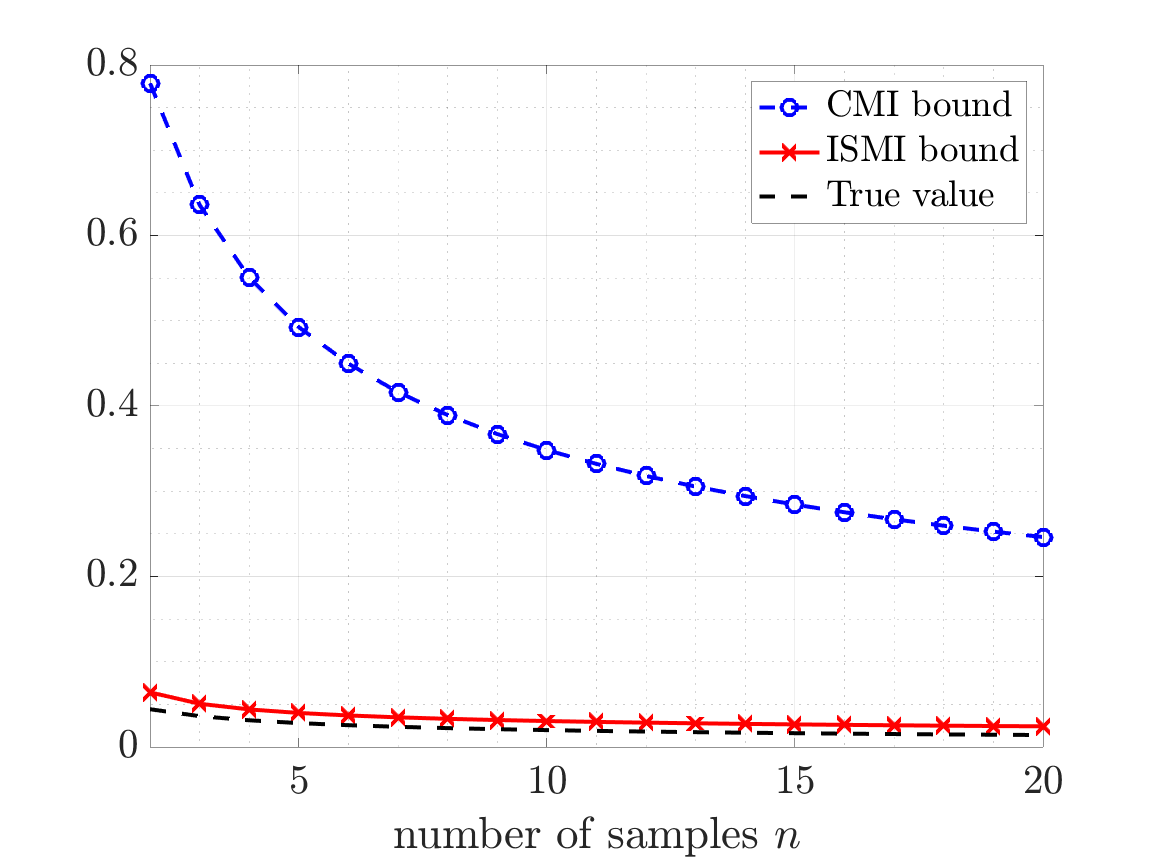}\\
	\caption{Comparison of different generalization bounds for the ERM algorithm with  additive noise $\epsilon=0.05$.}\label{fig:noisy}
\end{figure}

Numerical comparisons are presented in Fig. \ref{fig:ERM} and Fig. \ref{fig:noisy}. In both figures, we plot the ISMI bound, the CMI bound in  \cite{NIPS2018}, and the true values of the generalization error, as functions of the number of samples $n$. In Fig. \ref{fig:ERM}, we compare these bounds for the ERM solution $W$. Note that the CMI bound reduces to the classical chaining bound in this case. In Fig. \ref{fig:noisy}, we evaluate these bounds for the noisy algorithm $W'$ with $\epsilon=0.05$. Both figures demonstrate that the ISMI bound is closer to the true values of the generalization error, and outperforms the CMI bound significantly. More details about the computations of both bounds can be found in  Appendix \ref{app:GP}.

\section{Noisy, Iterative Algorithms}\label{sec:noisy}
In this section, we apply the ISMI bound in Theorem \ref{thm:main} to a class of noisy, iterative algorithms, specifically, stochastic gradient Langevin dynamics (SGLD).

%
\subsection{SGLD Algorithm}
We begin by introducing some notation to be used in this section. Denote the parameter vector at iteration $t$ by $W_{(t)} \in \mathbb{R}^d$, and let $W_{(0)} \in \mathcal{W}$ denote an arbitrary initialization. At each iteration $t \ge 1$, we sample a training data point $Z_{U_{(t)}} \in S$, where $U_{(t)}\in\{1,...,n\}$ denotes the random index of the sample selected at iteration $t$, and compute the gradient $\nabla \ell(W_{(t-1)},Z_{U_{(t)}})$. We then scale the gradient by a step size $\eta_{(t)}$ and perturb it by isotropic Gaussian noise $\xi\sim \mathcal{N}(0,I_d)$. The overall update rule is as follows \cite{welling2011bayesian}:
\begin{equation}
W_{(t)} = W_{(t-1)} - \eta_{(t)}  \nabla \ell(W_{(t-1)},Z_{U_{(t)}})+\sigma_{(t)}\xi,
\end{equation}
where $\sigma_{(t)}$ controls the variance of the Gaussian noise.

For $t \ge 0$, let $W^{(t)} \triangleq \{W_{(1)},\ \cdots,\ W_{(t)}\}$ and  $U^{(t)} \triangleq \{U_{(1)},\cdots,U_{(t)}\}$. We assume that the training process takes $K$ epochs, and the total number of iterations is $T=nK$. The output of the algorithm is $W=W_{(T)}$.

In the following, we use the same assumptions as in  \cite{Pensia2018ISIT}.
\begin{assum}
$\ell(w,Z)$ is $R$-sub-Gaussian with respect to $Z \sim \mu$, for every $w \in \mathcal{W}$.
\end{assum}

\begin{assum}\label{assum:bounded}
The gradients are bounded, i.e., $\sup_{w\in \mathcal{W}, z\in \mathcal{Z}} \|\nabla \ell(W,z)\|_2 \le L$, for some $L > 0$.
\end{assum}
In  \cite{Pensia2018ISIT}, the following bound was obtained by upper bounding $I(W;S)$ in Lemma \ref{lemma:original}.
\begin{lemma}[{\cite[Corollary 1]{Pensia2018ISIT}}]\label{lemma:SGLD}
The generalization error of the SGLD algorithm is bounded by
\begin{equation}
  |\mathrm{gen}(\mu, P_{W|S})| \le \sqrt{\frac{R^2}{n}\sum_{t=1}^T \frac{\eta_t^2L^2}{\sigma_t^2}}.
\end{equation}
\end{lemma}

\subsection{ISMI Bound for SGLD}
We have the following proposition which characterizes the ISMI bound for the SGLD algorithm.
\begin{proposition}\label{prop:SGLD}
Suppose Assumption 1 and 2 hold, then we have the following ISMI bound on the generalization error for SGLD algorithm,
\begin{align}
|\mathrm{gen}(\mu,P_{W|S})| 
\le \mathbb{E}_{U^{(T)}} \bigg[\frac{R}{n} \sum_{i=1}^n \sqrt{ \sum_{\tau \in \gT_i(U^{(T)})}  \frac{\eta_{(\tau)}^2 L^2}{  \sigma_{(\tau)}^2} } \bigg],
\end{align}
where $U^{(T)}$ denotes the random sample path, and $\gT_i(U^{(T)})$ denote the set of iterations for which samples $Z_i$ is selected for a given sample path $U^{(T)}$.
\end{proposition}
\begin{proof}
	To apply the ISMI bound for SGLD, we modify the result in Theorem \ref{thm:main} by conditioning on the random sample path $U^{(T)}$,
	\begin{align}\label{eq:SGLD_gen}
	&|\mathrm{gen}(\mu,P_{W|S})| \nn \\
	&  = \Big|\mathbb{E}_{U^{(T)}}  \Big[\frac{1}{n} \sum_{i=1}^n \Big(\mathbb{E}_{W,\widetilde Z}[\ell(W,\widetilde Z)|U^{(T)}] \nn\\
    &\qquad \qquad \qquad \qquad \qquad - \mathbb{E}_{W,Z_i} [\ell(W,Z_i)|U^{(T)}] \Big) \Big] \Big| \nn\\
	& \le \frac{1}{|\mathcal{U}|} \sum_{u^{(T)} \in \mathcal{U}} \Big(\frac{1}{n} \sum_{i=1}^n \sqrt{2R^2  I(W;Z_i| U^{(T)} = u^{(T)})}\Big),
	\end{align}
	where $\mathcal{U}$ denotes the set of all possible sample paths,
	and $I(W;Z_i|U^{(T)}= u^{(T)})$ is the mutual information \footnote{Note that this mutual information is different from the conditional mutual information $I(W;Z_i|U^{(T)}) = \mathbb{E}_{U^{(T)}}[I(W;Z_i|U^{(T)}= u^{(T)})]$.} with conditional distribution $P(W,Z_i| U^{(T)}= u^{(T)})$.
	
	
	Let $\gT_i(u^{(T)})$ denote the set of iterations for which sample $Z_i$ is selected for a given sample path $u^{(T)}$. Using the chain rule of mutual information, we have
	\begin{align}
	&I(W;Z_i| U^{(T)}= u^{(T)} ) \nn \\
	& \le  I(Z_i;  W^{(T)}|U^{(T)}= u^{(T)})\nn\\
	& = \sum_{\tau =1}^T I(Z_i; W_{(\tau)}|W_{(\tau-1)}, U^{(T)}= u^{(T)})\nn\\
	& = \sum_{\tau \in \gT_i(u^{(T)})} I(Z_i; W_{(\tau)}|W_{(\tau-1)}, U^{(T)}= u^{(T)}),
	\end{align}
	where the last equality is due to the fact that given $u^{(T)}$ and $W_{(\tau-1)}$, $Z_i$ is independent of $W_{(\tau)}$, if $\tau \notin \gT_i(u^{(T)})$. For $\tau \in \gT_i(u^{(T)})$, i.e., if $Z_i$ is selected at iteration $\tau$, we have
	\begin{align*}
	&I(Z_i; W_{(\tau)}|W_{(\tau-1)}, U^{(T)}= u^{(T)})\nn \\
	&= h\big(\eta_{(\tau)} \nabla \ell(W_{(\tau-1)},Z_{i})+\sigma_{(\tau)}\xi |W_{(\tau-1)} \big)
	- h(\sigma_{(\tau)}\xi).
	\end{align*}
	Since we assume that $\sup_{w\in \mathcal{W}, Z\in \mathcal{Z}} \|\nabla \ell(W,Z)\|_2 \le L $, we have
	\begin{align}
	&h\big(\eta_{(\tau)} \nabla \ell(W_{(\tau-1)},Z_{i})+\sigma_{(\tau)}\xi |W_{(\tau-1)} \big) \nn \\
    & \le h\big(\eta_{(\tau)} \nabla \ell(W_{(\tau-1)},Z_{i})+\sigma_{(\tau)}\xi \big)\nn\\
	& \le \frac{d}{2} \log \big(2 \pi e \frac{\eta_{(\tau)}^2 L^2+d\sigma_{(\tau)}^2}{ d} \big).
	\end{align}
	Due to the fact that $\xi$ is an independent Gaussian noise,  $h(\sigma_{(\tau)}\xi |W_{(\tau-1)})= \frac{d}{2}\log \big(2 \pi e \sigma_{\tau}^2\big)$, we have
	\begin{align*}
	I(Z_i; W_{(\tau)}|W_{(\tau-1)},U^{(T)}= u^{(T)}) \le \frac{d}{2} \log \big(1+ \frac{\eta_{(\tau)}^2 L^2}{ d \sigma_{(\tau)}^2} \big).
	\end{align*}
	
	Combining with \eqref{eq:SGLD_gen}, it follows that
	\begin{align}
	|\mathrm{gen}(\mu,P_{W|S})| 
	\le \mathbb{E}_{U^{(T)}} \bigg[\frac{R}{n} \sum_{i=1}^n \sqrt{ \sum_{\tau \in \gT_i(U^{(T)})}  \frac{\eta_{(\tau)}^2 L^2}{  \sigma_{(\tau)}^2} } \bigg],
	\end{align}
	where we remove the $\log$ term by using $\log(1+x)\le x$.
\end{proof}

To compare the result of the ISMI bound in Proposition \ref{prop:SGLD} and the bound in Lemma \ref{lemma:SGLD}, we  specify the parameters in the SGLD algorithms.
As in \cite{Pensia2018ISIT}, we set $\eta_{(t)} = \frac{c}{t}$, and $\sigma_{(t)} = \sqrt{\eta_t}$. We use the following ``without replacement'' sampling scheme for SGLD to further simplify the computation. Specifically, for the $k$-th training epoch, i.e., from the $((k-1)n+1)$-th to $kn$-th iterations, all training samples in $S$ are used exactly once.

Then, the ISMI bound can be further bounded as follows:
\begin{align}
  &|\mathrm{gen}(\mu,P_{W|S})|\nn \\
  & \le \frac{R L }{n} \mathbb{E}_{U^{(T)}}\Big[ \sum_{i=1}^n \sqrt{ \sum_{\tau \in \gT_i(U^{(T)})}  \frac{c}{ \tau} }  \Big] \nn\\
  & \overset{(a)}{\le} \frac{R L\sqrt{c}}{n} \sum_{i=1}^n \sqrt{ \frac{1}{i} + \sum_{k=1}^{K-1}  \frac{1}{ nk} }\nn\\
  & \overset{(b)}{\le} \frac{R L \sqrt{c}}{n} \sum_{i=1}^n \sqrt{ \frac{1}{i} + \frac{ \log(K-1)+1}{n} }\nn\\
  & \overset{(c)}{\le} \frac{R L }{\sqrt{n}} \Big(\sqrt{{c\log(K-1)+c}} + o(\log\log K) \Big),
\end{align}
where $(a)$ follows from the sampling scheme that all samples are used exactly once in each epoch;  $(b)$ is due to the fact that $\sum_{k=1}^K \frac{1}{k} \le \log (K) +1$; and $(c)$ follows by computing the integral $\int_0^1 \sqrt{ \frac{1}{x} + 1+ \log(K-1) } dx$.

Comparing with the bound in \cite{Pensia2018ISIT},
\begin{equation}
  |\mathrm{gen}(\mu, P_{W|S})| \le \frac{R L }{\sqrt{n}} \sqrt{{c\log(nK)+c}},
\end{equation}
it can be seen that our bound is tighter by a factor of $\sqrt{\log n}$ with the ``without replacement'' sampling scheme.

\begin{remark}
We note that for the typical use of SGLD, the standard deviation of the noise is  $\sigma_t = \sqrt{2 \eta_t/\beta_t}$, where $\beta_t$ denotes the inverse temperature at iteration $t$, and it is often set to be $\Theta(n)$. It is clear that $\beta=\Theta(n)$ will lead to a generalization bound that does not decay in $n$. Here, we choose $\beta_t = 2$ for comparison with the bound in \cite{Pensia2018ISIT}, while in practice $\beta$ may be a function of $n$ and grow with $t$. An analysis of the generalization error bound of SGLD for arbitrary choices of $\beta_t$ can be found in \cite{negrea2019information}.
\end{remark}

\section{Empirical Evaluation of ISMI Bound for Logistic Regression}
\label{sec:eva}
For some learning algorithms applied in practice, it is difficult to analytically characterize  $P_{W|S}$, which makes the analytical evaluation of the ISMI bound challenging. In this section, we provide such an example, that of  logistic regression, for which it is difficult to analytically characterize the learning algorithm via the conditional distribution $P_{W|S}$. We therefore empirically evaluate the ISMI bound via a mutual information estimator, and compare it to an empirical evaluation of the generalization error. We further note that the ISMI bound is much  easier to estimate than the bound in Lemma \ref{lemma:original} and the chaining bound in \cite{NIPS2018} due to the significant reduction in dimension that comes from estimating $I(Z_i;W)$ instead of $I(S;W)$.


Consider the binary classification problem, where the samples $Z=(X,Y)$, consisting of features $X\in \mathbb{R}^d$ and labels $Y \in \{\pm 1\}$. We assume that training samples are generated from the following distribution,
\begin{equation}
    X\sim \mathcal{N}(\mu_{Y},\Sigma),\quad Y \in \{\pm 1\}, \quad \mu_{Y}\in \mathbb{R}^d.
\end{equation}
The marginal distribution of $X$ is the mixture of two Gaussian distributions, and we assume that $P(Y=-1)=P(Y=1)=1/2$.

A binary classifier is constructed as follows:
\begin{align}
  \hat{Y} = \left\{
              \begin{array}{ll}
                1, & \hbox{ $w^T X \ge 0$;} \\
                -1, & \hbox{else.}
              \end{array}
            \right.
\end{align}
We adopt classification error $\ell(w,Z)= \mathds{1}_{\{Y\ne \hat{Y}\}}$ to compute the generalization error, then the empirical risk with $n$ i.i.d. samples is
\begin{equation}
  L_S(w) = \frac{1}{n}\sum_{i=1}^n \mathds{1}_{\{Y_i \ne \hat{Y_i}\}}.
\end{equation}

Since the empirical risk function is not differentiable,  we learn $W$  by minimizing the following loss function of logistic regression:
\begin{equation}\label{eq:logistic_loss}
  W = \argmin_{w\in \mathcal{W}} \frac{1}{n} \sum_{i=1}^n \log(1+e^{-Y_i w^T X_i}).
\end{equation}
In general, it is difficult to obtain a  closed form solution for this optimization problem, and therefore, \eqref{eq:logistic_loss} is usually solved  numerically. This makes it difficult  to analytically characterize the conditional distribution $P_{W|Z_i}$, which in turn makes it challenging to compute the generalization error and the ISMI bound analytically.

Alternatively, we can empirically estimate the generalization error and the ISMI bound. Specifically, we train $W$ for $N$ times using $N$ sets of independent samples, and we use the K-nearest neighbor based mutual information estimator  \cite{kraskov2004estimating,gao2018demystifying} to estimate $I(W;Z_i)$  with $N$ i.i.d. samples of $W$ and $Z_i$. Note that the K-nearest neighbor based mutual information estimator is consistent, and its  mean squared estimation error can be upper bounded by $\mathcal O(N^{-\frac{2}{d_W+d_Z}})$, where $d_W=d$ is the dimension of the weights $W$, and $d_Z=d+1$ is the dimension of $Z$  \cite{gao2018demystifying}. Moreover, since we use classification error to compute generalization error, $\ell(W,Z)$ is bounded by 1. Then, by Hoeffding's lemma,
$\ell(W,Z)$ is $\frac{1}{2}$-sub-Gaussian. Thus, the ISMI bound can be estimated by
\begin{equation}\label{eq:eva}
 \frac{1}{n} \sum_{i=1}^n \sqrt{\frac{\hat{I}(W;Z_i)}{2}},
\end{equation}
where $\hat{I}(W;Z_i)$ is the estimate of $I(W;Z_i)$. If we apply an optimization algorithm that does not depend on the order of the samples, e.g., gradient descent and stochastic gradient descent with random shuffling, we then only need to estimate one
$\sqrt{\frac{\hat{I}(W;Z_1)}{2}}$ instead of estimating $\hat{I}(W;Z_i)$ for all $1\leq i\leq n$.

\begin{figure}
  \centering
  \includegraphics[width=0.95 \columnwidth]{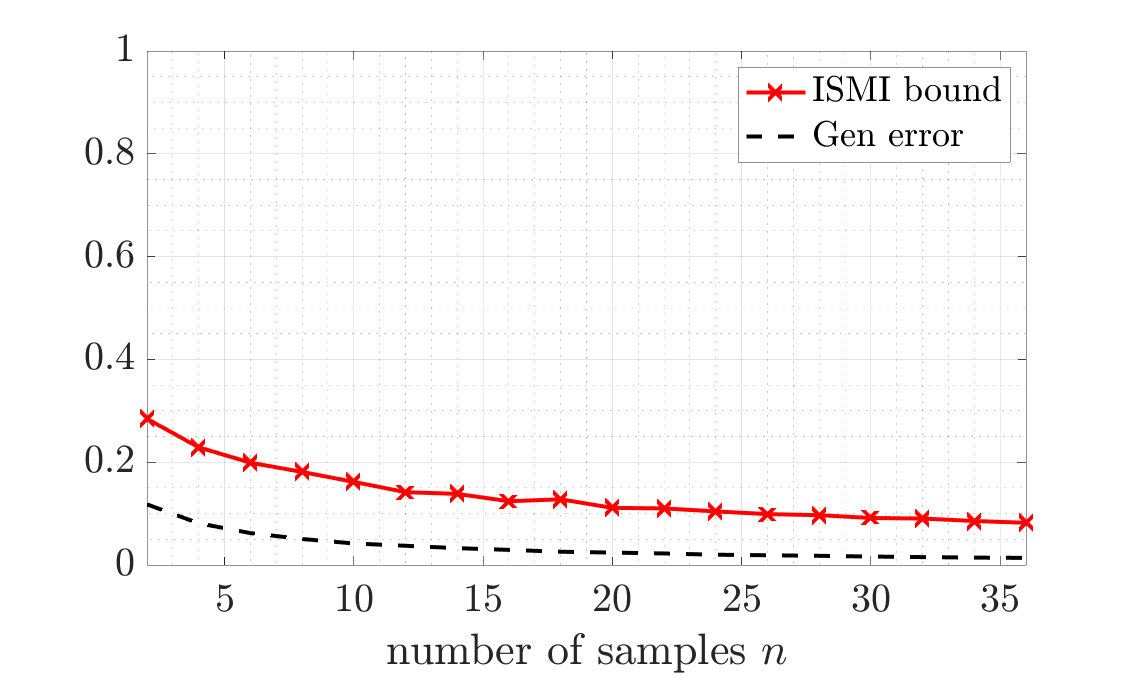}\\
  \caption{Empirical evaluation of ISMI bound and generalization error in Logistic regression.}\label{fig:LR}
\end{figure}

We note that the bound in Lemma \ref{lemma:original} is difficult to estimate due to the high dimension of $S$, which scales linearly with $n$. Specifically, the training dataset $S$ consists of $n$ samples, and therefore  $I(W;S)$ is the mutual information between two random vectors with dimensions $d$ and $n(d+1)$. As shown in \cite{gao2018demystifying}, due to the curse of dimensionality, it is impossible to construct a consistent mutual information estimator for large $n$. We also note that the exact computation of Wasserstein distances is costly in general, as it requires the solution of an optimal transport problem\cite{rowland2019orthogonal}. Moreover, similar high dimensional issue makes it even more difficult to directly estimate $\mathbb{W}(P_W,P_W|S)$ in the Wasserstein distance based generalization bound in \cite{lopez2018generalization,ISIT2019wass}. 

In  Fig. \ref{fig:LR}, we plot an empirical estimate of the ISMI bound using \eqref{eq:eva}, and compare it to the generalization error. In the simulation, we chose the following model parameters: $d=2$ and $\mu_1=(1,1)$, $\mu_{-1}=(-1,-1)$ with $\Sigma=4I$. We used the K-nearest neighbor based mutual information estimator (revised KSG estimator) in \cite{gao2018demystifying} with $N=5000$ i.i.d. samples. It can be seen that the ISMI bound has a similar convergence behavior as the true generalization error as number of training samples $n$ increases.

\section{Conclusions}
In this paper, we proposed a tighter information-theoretic upper bound on the generalization error using the mutual information $I(Z_i;W)$ between each individual training sample $Z_i$ and the output hypothesis $W$ of the learning algorithm. We showed that compared to existing studies, our bound is more broadly applicable, and is considerably tighter. More importantly, the individual sample mutual information is between two vectors whose dimensions do not scale with the sample size $n$. Therefore, unlike the existing bounds  in \cite{NIPS2018,ISIT2019wass,xu2017information},  the ISMI bound can easily be evaluated empirically in practice.
As suggested by recent works, the proposed ISMI bound could be further improved by combining with the  chaining method \cite{asadi2019chaining}, or data-dependent estimates \cite{negrea2019information}. The proposed information-theoretic framework can also be used to guide model compression in  deep learning \cite{bu2019information}.



\appendices
\section{Section \ref{sec:mean} Details} \label{app:mean}
\subsection{Generalization Error}
For this example, the generalization can be computed as
\begin{align}
  &\mathrm{gen}(\mu,P_{W|S}) \nn \\
  & =\mathbb{E}_{W, S}[L_\mu(W)-L_S(W)] \nn \\
   & =  \mathbb{E}_{S}\Big[\mathbb{E}_{\widetilde{Z}}[\| \bar{Z}-\widetilde{Z}\|_2^2] - \frac{1}{n} \sum_{i=1}^n \| \bar{Z}-Z_i\|_2^2 \Big] \nn \\
   & = \mathbb{E}_{S,\widetilde{Z}}[{\rm{Tr}}((\bar{Z}-\widetilde{Z})(\bar{Z}-\widetilde{Z})^\top)] \nn \\
   & \qquad - \mathbb{E}_{S}\Big[\frac{1}{n} \sum_{i=1}^n {\rm{Tr}}((\bar{Z}-Z_i)(\bar{Z}-Z_i)^\top)\Big]\nn \\
   & ={\rm{Tr}}({\rm{Cov}}[\bar{Z}])+{\rm{Tr}}({\rm{Cov}}[\widetilde{Z}]) - \frac{n-1}{n}{\rm{Tr}}({\rm{Cov}}[Z]) \nn \\
   & = \frac{2}{n} {\rm{Tr}}({\rm{Cov}}[Z])
\end{align}
Since ${\rm{Cov}}[Z]= \sigma^2I_d$, we have $\mathrm{gen}(\mu,P_{W|S}) =\frac{2\sigma^2 d}{n}$.

\subsection{Individual Mutual Information}
We note that both $W$ and $n$ i.i.d. samples $Z_i$ are Gaussian, then the individual mutual information $I(W;Z_i)$ is a function captured by the following covariance matrix,
\begin{equation}
  {\rm{Cov}}[W,X_i] = \left(
    \begin{array}{cc}
      \Sigma/n & \Sigma/n \\
      \Sigma/n & \Sigma \\
    \end{array}
  \right).
\end{equation}
Then, we have
\begin{align}
  I(W;X_i) &= \frac{1}{2} \log \frac{|{\rm{Cov}}[W]| |{\rm{Cov}}[X_i]| }{ |{\rm{Cov}}[W,X_i]| } \nn\\
  & = \frac{1}{2} \log \frac{ |\Sigma/n| |\Sigma| }{ |\frac{n-1}{n^2}\Sigma| |\Sigma|  } \nn \\
  & = \frac{d}{2} \log \frac{n}{n-1},
\end{align}
for all $i=1,\dots,n$.

\subsection{Upper bound for CGF}
Note that the CGF of $\ell(\widetilde{W},\widetilde{Z})$ is given by
\begin{align}
\Lambda_{\ell(\widetilde W,\widetilde Z)}(\lambda) &= - d \sigma_\ell^2 \lambda - \frac{d}{2} \log(1-2\sigma_\ell^2\lambda) \nn \\
& = \frac{d}{2} (-u-\log(1-u)), \  \lambda \in (-\infty, \frac{1}{2\sigma_\ell^2}),
\end{align}
where $u\triangleq2\sigma_\ell^2\lambda$. Further note  that
\begin{equation}
-u-\log(1-u) \le \frac{u^2}{2},\ u<0.
\end{equation}
We therefore have the following upper bound on the CGF of $\ell(\widetilde{W},\widetilde{Z})$:
\begin{equation}
\Lambda_{\ell(\widetilde W,\widetilde Z)}(\lambda) \le d\sigma_\ell^4\lambda^2,\quad \lambda<0.
\end{equation}

\section{Section \ref{sec:GP} Details}\label{app:GP}
\subsection{Generalization Error}
Note that the expectation of the population risk is
\begin{align}
 \mathbb{E}_{W, S}[L_\mu(W)] = \mathbb{E}_{W,Z}[-\langle W,Z \rangle]=0,
\end{align}
since $W$ and $Z$ are independent. Then, the generalization error can be computed as
\begin{align}\label{equ:gen_GP}
&\mathrm{gen}(\mu,P_{W|S}) = \mathbb{E}[-L_S(W)] \nn \\
&= \mathbb{E}_{W, S}  [ \langle W,  \frac{1}{n}\sum_{i=1}^n Z_i\rangle]\nn \\
&=  \mathbb{E}_{W, S}    \Big\| \frac{1}{n}\sum_{i=1}^n Z_i \Big\|_2 =\sqrt{\frac{\pi}{2n}},
\end{align}
where the last step is due to the fact that the distribution of $\| \frac{1}{n}\sum_{i=1}^n Z_i \|_2$ is $\mathrm{Rayleigh}(\frac{1}{n})$.

\subsection{Individual Sample Mutual Information Bound}
To compute the ISMI bound, we need the conditional distribution $P_{W|Z_i=z_i}$.
Note that given $Z_i=z_i$, the ERM solution is
\begin{equation*}
W = \argmax_{\phi\in [0,2\pi)} \langle w,  \frac{z_i}{n}+\frac{1}{n}\sum_{j\ne i} Z_i\rangle.
\end{equation*}
Also note that since $\phi\in [0,2\pi)$,  $P_{W|Z_i=z_i}$ is equivalent to the phase distribution of a Gaussian random vector $\mathcal{N}(\frac{z_i}{n}, \frac{n-1}{n^2}I_2)$ in polar coordinates. 
Since entropy is shift-invariant, we can always rotate the polar coordinates, such that $z_i=(r,0)$, where $r \in \mathbb{R}^+$ is the $\ell_2$ norm of $z_i$.

The joint distribution of  radius and phase $(\rho,\phi)$ in polar coordinates can be obtained by applying the Jacobian method to the Gaussian distribution $\mathcal{N}(\frac{(r,0)}{n}, \frac{n-1}{n^2}I_2)$, and we have
\begin{align}
f\big(\rho, \phi\ \big|\|Z_i\|=r\big) = \frac{n^2\rho}{2\pi(n-1)}e^{-\frac{n^2\rho^2+r^2}{2(n-1)}}e^{-\frac{nr\rho \cos \phi}{(n-1)}},
\end{align}
for $\rho\in [0,\infty),\ \phi \in [0,2\pi)$.

Then, the marginal distribution of $\phi$ can be computed by integrating out $\rho$ from the joint distribution $f\big(\rho, \phi\big|\|Z_i\|=r\big)$:
\begin{align*}
&f\big(\phi\big|\|Z_i\|=r\big) \nn \\
&= \int_{0}^\infty f\big(\rho, \phi\ \big|\|Z_i\|=r\big)d\rho \nn \\
&= \frac{1}{2 \pi} e^{-\frac{r^2}{2(n-1)}}
+\frac{r \cos \phi}{\sqrt{2\pi (n-1)}}e^{-\frac{r^2\sin^2\phi}{2(n-1)}}Q(-\frac{r\cos \phi}{n-1}),
\end{align*}
where $Q(x)$ is the complementary cumulative distribution function of the standard normal distribution.

Thus, the ISMI bound for the ERM algorithm $W$ can be evaluated using the following expression via numerical integration:
\begin{align}
  &|\mathrm{gen}(\mu,P_{W|S}) |\le \sqrt{2 I(W;Z_i) } \nn \\
  &=\sqrt{2\log 2\pi - 2\mathbb{E}_{\|Z_i\|}\Big[h\big(f(\phi\big|\|Z_i\|=r)\big)\Big]}.
\end{align}
Similarly, the ISMI bound for algorithm $W'$ can be computed via numerical integration using
\begin{align}
  &|\mathrm{gen}(\mu,P_{W'|S}) |\le \sqrt{2 I(W';Z_i) } \nn \\
  &=\sqrt{2\log 2\pi - 2\mathbb{E}_{\|Z_i\|}\Big[h\big(f(W'\big|Z_i=z_i)\big)\Big]},
\end{align}
where the conditional distribution $P_{W'|Z_i=z_i}$ can be characterized by the phase distribution
\begin{align}
  &f(\phi'\big|\|Z_i\|=r) \nn \\
  &= \frac{1 - \epsilon}{2\pi} + \frac{\epsilon}{2 \pi} e^{-\frac{r^2}{2(n-1)}} \nn \\
    &\qquad +\frac{\epsilon r \cos \phi'}{\sqrt{2\pi (n-1)}}e^{-\frac{r^2\sin^2\phi'}{2(n-1)}}Q(-\frac{r\cos \phi'}{n-1}).
\end{align}

\subsection{Chaining Mutual Information Bound}

The CMI bound is computed based on the values provided in Table 1 in \cite{NIPS2018}. We note that the CMI  bound in \cite{NIPS2018} is evaluated for the case $n=1$, i.e., there is only one training sample. To plot the CMI bound in \cite{NIPS2018} as a function of the number of samples $n$, we normalize the CMI bound by a $\sqrt{n}$ factor in Figures \ref{fig:ERM} and \ref{fig:noisy}, since $\mathcal O(1/\sqrt{n})$ is the true convergence rate for the generalization error as shown in \eqref{equ:gen_GP}. For instance, Table 1 in \cite{NIPS2018} shows that the CMI bound for the ERM solution $W$ is 19.0352 when $n=1$, which is equivalent to the classical chaining bound. We therefore plot the curve $\frac{19.0352}{\sqrt{n}}$ as the CMI bound  for  comparison with the proposed ISMI bound in Figure \ref{fig:ERM}.

\bibliographystyle{ieeetr}
\bibliography{generalization}

\end{document}